\documentclass[]{article}
\usepackage{amsmath,amsfonts,amsthm,lineno,graphicx,subfigure}
\usepackage[]{hyperref}
\hypersetup{pdftitle={On Convergent Finite Difference Schemes for Variational - PDE Based Image Processing},
colorlinks=true,citecolor=blue,linkcolor=red}
\newcommand{\abs}[1]{\left\vert #1 \right\vert}
\newcommand{\norm}[1]{\left\Vert#1\right\Vert}

\newtheorem{definition}{Definition}
\newtheorem{theorem}{Theorem}
\newtheorem{lemma}{Lemma}
\newtheorem{remark}{Remark}

\begin{document}

\title{On Convergent Finite Difference Schemes for Variational - PDE Based Image Processing\thanks{
The work was initiated at the First Central Region Conference on Numerical Analysis and Dynamical Systems (CRCNADS), University of Kansas, Lawrence, KS, USA, May 3--5, 2013. Download the poster at figshare: \url{http://dx.doi.org/10.6084/m9.figshare.695306}.}}

\author{V. B. Surya Prasath\thanks{Corresponding author. Department of Computer Science, University of Missouri-Columbia, MO 65211 USA. E-mail: prasaths@missouri.edu}\and Juan C. Moreno\thanks{IT, Department of Computer Science, University of Beira Interior, 6201--001, Covilh\~{a}, Portugal. E-mail: jmoreno@ubi.pt.}} 
\date{}
\maketitle

\begin{abstract}

We study an adaptive anisotropic Huber functional based image restoration scheme. By using a combination of L2-L1 regularization functions, an adaptive Huber functional based energy minimization model provides denoising with edge preservation in noisy digital images. We study a convergent finite difference scheme based on continuous piecewise linear functions and use a variable splitting scheme, namely the Split Bregman~\cite{GO09}, to obtain the discrete minimizer. Experimental results are given in image denoising and comparison with additive operator splitting, dual fixed point, and projected gradient schemes illustrate that the best convergence rates are obtained for our algorithm. 

\end{abstract}
\section{Introduction}\label{sec:intro}
\linenumbers

Variational and partial differential differential equations (PDEs) based schemes are popular in image and video processing problems. In particular in image restoration, adaptive edge preserving smoothing can be achieved by choosing regularizing functions or equivalently diffusion coefficients carefully. This has been the object of study for the last three decades and we mention the seminal work of  Perona  and Malik~\cite{PM90} as the starting point in PDE based image processing and the connections to variational and robust statistics has also been considered later~\cite{BR96,CG98,WE98}. We refer to the recent monographs~\cite{AK06,Scherzerbook09} for an overview of these methods.

Based on the smoothness or regularity assumptions on the true image, various regularization functions can be used. The Tikhonov regularization function~\cite{TA77} which is based on the quadratic growth, $L^2$-gradient minimization, suppresses gradients and thus is effective in removing noise. Unfortunately gradients can also represent edges which are important for further pattern recognition tasks. To avoid the over smoothing total variation or the $L^1$-gradient minimization, which is widely known as the total variation (TV) regularization model, has been advocated~\cite{RO92}. Recently, there are efforts to combine both the $L^2$ and $L^1$ based fundtionals into one common minimization problem such as the Huber function~\cite{BS98,PSd10}, inf-sup convolution~\cite{CL97,CasellasSapiro00}. Adaptive versions of the variational - PDE models are gaining popularity~\cite{Che05,CL06,PSc10,PSe12,Prasath11,PVorotnikov12} and can give better restoration results than non-adaptive schemes in terms of edge preservation. The discrete approximation to the continuous variational - PDE schemes from image processing using finite difference and finite element based schemes have been studied~\cite{Ch95,DobsonVogel97,WR98,ChanMulet99,Ch99,SpitaleriMarchmultigridANM01,JiaZhao09,WeissBF09,WuYangPangfourthfixedpointANM12}. Convergence of finite differences for various PDEs is a classic area within numerical analysis\footnote{Semen Aronovich Ger\v{s}gorin's work~\cite{Gersgorin30} in 1930 was the first paper to treat the important topic of the convergence of finite-difference approximations to the solution of Laplace-type equations.} 
and is still an active area of research in application areas such as image processing~\cite{ChambolleLevineFDsiims11,LaiROFdiscrete12,CarliniFerrettiMeancurvdiscANM12,XuGeomdiffconvANM13}. 

In this paper we consider convergent finite difference schemes for an adaptive Huber type functional based energy minimization model. We provide comparison with other convex variational regularization functions and use an edge indicator function guided regularization model. By using piecewise continuous linear functions along with the discrete energy we study the convergence of discrete minimizer to the continuous solution. To solve corresponding discrete convex optimization problem various solvers exist, such as the dual minimization~\cite{Ch04}, primal-dual~\cite{ChambollePock11} alternating direction method of multipliers and, operator splitting~\cite{Setzer11} etc. Here we use the split Bregman method studied by Goldstein and Osher~\cite{GO09,GB10} for computing the discrete energy minimizer as it is the fastest in terms of computational complexity and then prove a convergence result for the class of weakly regular images. We utilize an image adaptive inverse gradient based regularization parameter for better denoising without destroying salient edges. Experimental results on real and synthetic noisy images are given to highlight the noise removal property of the proposed model. Comparison results with different discrete optimization models in undertaken and further visualization are provided to support split Bregman based solution. 

The rest of the paper is organized as follows. Section~\ref{sec:comb} provides the background on an adaptive Huber variational - PDE model along with some basic results on bounded variation space. Section~\ref{sec:conv} details a convergent numerical scheme for the variational scheme. Section~\ref{sec:exper} provides comparative numerical results on noisy images and Section~\ref{sec:conc} concludes the paper.
\section{Continuous $L^2$-$L^1$ variational - PDE model}\label{sec:comb}

\begin{figure}
\centering
    \subfigure[Regularizers $\varphi(s)$]{\includegraphics[width = 6.75cm, height=5cm]{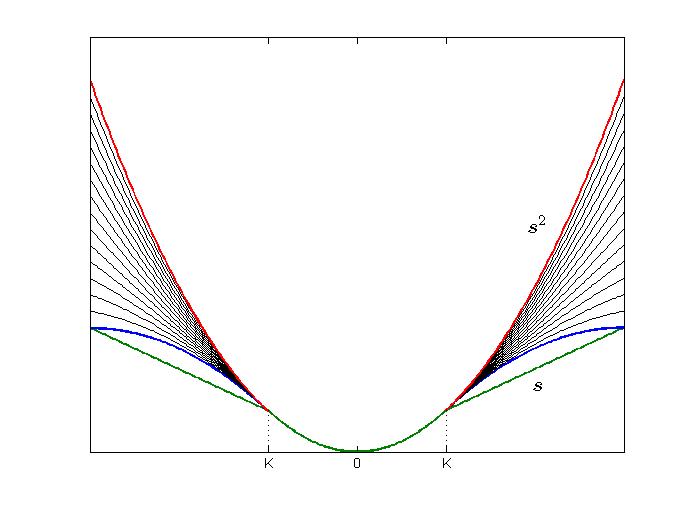}}
    \subfigure[Diffusivities $g(s)=\varphi'(s)/2s$]{\includegraphics[width = 6.75cm, height=5cm]{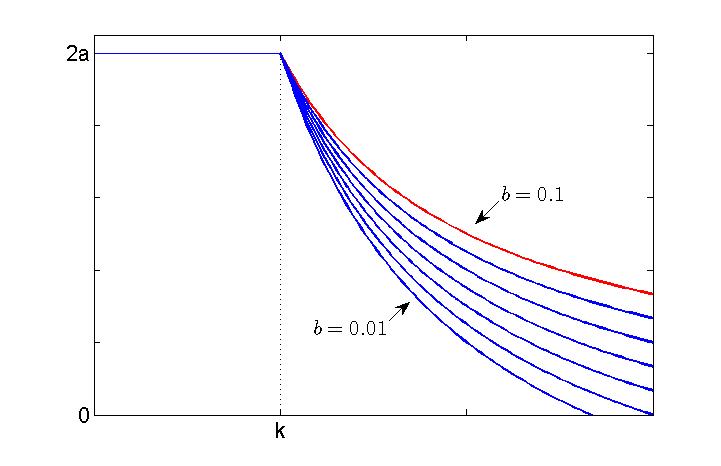}}
    \caption{Regularization and diffusion functions.
    (a) The regularization function $\varphi_{S}$~(\ref{E:ourm}) lies between the quadratic curve $s^2$ and linear $s$ when $\abs{s} > k$ depending on the parameter $0<b<1$. 
    (b) Corresponding diffusion functions $g$.}\label{I:mphis}
\end{figure}
Let $u_0:\Omega\subset\mathbb{R}^2\to\mathbb{R}$ be the input (noisy\footnote{We assume Gaussian noise, i.e., $n\sim \mathcal{N}(0,\sigma_n)$.}) image. 
We consider the following continuous variational-PDE scheme for image restoration\footnote{Note we use the notation $\nabla $ to denote the gradient and in the space of bounded variation functions $BV$ it is infact a Radon measure and is understood in the sense of distributions. The equality $\int_{\Omega} \abs{Du} = \int_{\Omega} \abs{\nabla u}\,dx$ is true when $u\in W^{1,1}(\Omega)$.},
\begin{eqnarray}\label{E:regl2l1}
\min_{u\in BV(\Omega)} E(u) = \int_\Omega \phi(\mathbf{x},\abs{\nabla u})\,d\mathbf{x} + \frac{\lambda}{2}\int_\Omega \abs{u-u_0}^2\,d\mathbf{x}
\end{eqnarray}
The corresponding PDE can be written in term of the Euler-Lagrange equation,
\begin{eqnarray}\label{E:pdephi}
\frac{\partial u}{\partial t} = div\left(\frac{\phi'(\mathbf{x},\abs{\nabla u})\nabla u}{\abs{\nabla u}}\right) - \lambda \,(u-u_0)
\end{eqnarray}
The adaptive discontinuity function $\phi(\cdot,\abs{\nabla u(x)}) = W(\cdot)\times \varphi(\abs{\nabla u(x)})$ is chosen to be an even
function. Note that the PDE in Eqn.~\eqref{E:pdephi} is a generalized Perona and Malik~\cite{PM90}
\begin{eqnarray}\label{E:pmeqn}
\frac{\partial u}{\partial t} = div\left(g(\abs{\nabla u})\nabla u\right) - \lambda \,(u-u_0),
\end{eqnarray}
where the diffusion function $g$ is related with $\varphi'(s)
= 2s g(s)$. The diffusion coefficient function $g(\cdot)$ decides
how much smoothness occurs and helps in noisy pixels (outlier) rejection. Various choices for choosing
$\varphi$ exists in the literature,
see~\cite{GG84,GM87,Li95,CS05} and~\cite{SZ08} for a recent
review. Note that under Gaussian noise assumption the data fidelity term  (also called the likelihood
term) in Eqn.~\eqref{E:regl2l1} is quadratic and hence
convex in $u$. Thus, if the regularization term is also convex in $u$ then
we are guaranteed of the well-posedness of the energy minimization
scheme given in~\eqref{E:regl2l1}. There are functions which are
non-convex~\cite{GM87,Li95,BS98,RM03} with $\varphi(s)\sim s^2$ near
$0$ and asymptotically linear as $\abs{s}\to+\infty$. This can
cause unstable behavior as the scheme can be plagued with local
minima. In this paper, we concentrate on convex regularization functions and study a stable and convergent scheme.
\begin{remark}
There are other ways to incorporate adaptive weights inside the regularization function or equivalently the diffusion coefficient. For example, as in adaptive total variation, i.e., with $\varphi(s) = s $, $\phi(x,\abs{\nabla u(x)}) =  \abs{W(x)\cdot\nabla u(x)}$ or in general $\phi(\cdot,\abs{\nabla u(x)}) = \varphi( W(\cdot)\abs{\nabla u(x)})$. The main difference lies in the way the regularization function $\varphi$ is weighted anisotropically and the final results change according to the formulation utilized. The main convergence result in Section~\ref{sec:conv} holds true for these type of adaptive functions as well.
\end{remark}

Two of the most obvious choices for the regularization function $\varphi$ are the Tikhonov or $L^2$-gradient $\varphi(s) = s^2$ and the total variation (TV) or $L^1$-gradient $\varphi(s)=s$, see Figure~\ref{I:mphis}(a). Both these functions have their advantages and drawbacks as illustrated by a synthetic noisy step image restoration example given in Fig.~\ref{fig:step}. To further highlight the smoothing properties we show in Figure~\ref{fig:oned} a line taken across the $Step$ image and corresponding results\footnote{Evolution of the $Step$ edge synthetic image mesh under different schemes are available as movies in the supplementary material.}. The Tikhonov regularization though effective in removing noise, penalizes higher gradients and hence can smooth the step edge excessively as can be seen in the resultant Fig.~\ref{fig:step}(c). On the other hand the TV regularization better preserves the edges but some additional regions in the homogeneous parts can be enhances which is known as `staircasing' artifact, see Fig.~\ref{fig:step}(d). Hence, a robust regularizer is required for effective smoothing for denoising while edges are preserved. For example, motivated from the robust statistics, we consider the classical M-estimators Huber's min-max function~\cite{HU81} and the Tukey's bisquare function~\cite{TU77} which are given by,
\begin{align}
\varphi_{H}(s)& =\begin{cases}
                s^2/2 & \text{if} \abs{s} < k,\\
                 k (\abs{s}-\dfrac{k}{2}) &\text{if}\abs{s}>k,
            \end{cases}\label{E:huber}\\
            \varphi_{T}(s)& =\begin{cases}
                \frac{k^2}{6}\left(1- [1 - s^2/k^2]^3\right) &\text{if} \abs{s} < k,\\
                 \frac{k^2}{6} &\text{if}\abs{s}>k,
            \end{cases}\label{E:tukey}
\end{align}
respectively. Note that the parameter $k>0$ determines the region of transition between low and high gradients thereby providing a separation of homogeneous (flat) regions and edges (jumps).  To study the fine properties of the Huber and Tukey regularization functions on the final restoration result, we consider a simple 1D signal which consist of a sharp peak like edge and ramp edges along with flat regions.

\begin{figure}
\centering

\includegraphics[width=3.cm, height=2.8cm]{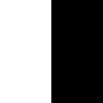}
\includegraphics[width=3.cm, height=2.8cm]{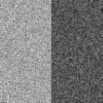}
\includegraphics[width=3.cm, height=2.8cm]{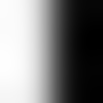}
\includegraphics[width=3.cm, height=2.8cm]{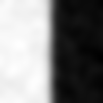}
\includegraphics[width=3.cm, height=2.8cm]{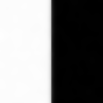}\\
	\subfigure[Original, $\sigma_n=20$]{\includegraphics[width=3.cm, height=2.8cm]{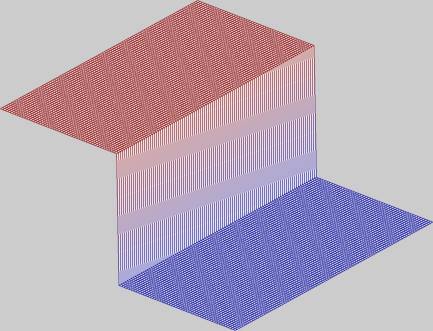}}
	\subfigure[Noisy, $\sigma_n=20$]{\includegraphics[width=3.cm, height=2.8cm]{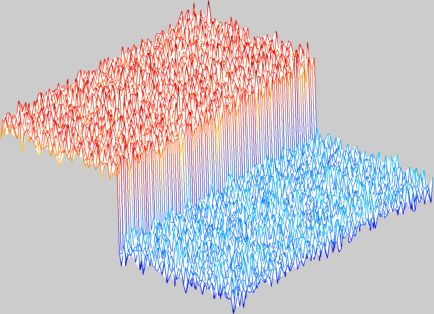}}
	\subfigure[Tikhonov]{\includegraphics[width=3.cm, height=2.8cm]{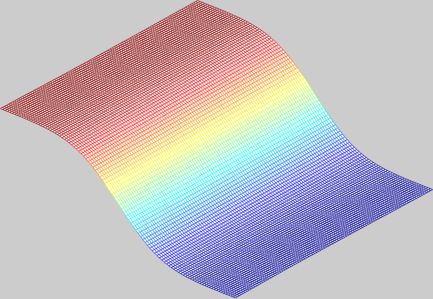}}
	\subfigure[TV]{\includegraphics[width=3.cm, height=2.8cm]{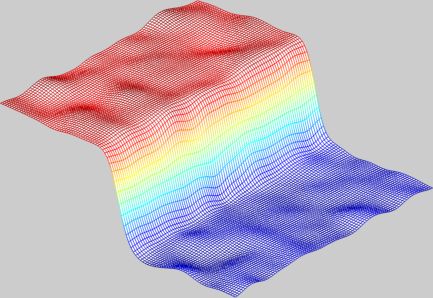}}
	\subfigure[Our]{\includegraphics[width=3.cm, height=2.8cm]{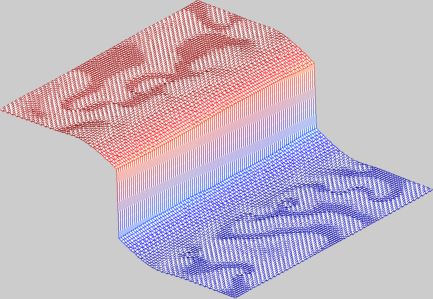}}
	\caption{Synthetic $Step$ image showing the effects of the choice of regularization function on the final restoration results. 
	The $L^2$ - gradient scheme (Tikhonov) over-smoothes the edge whereas $L^1$ - gradient scheme (TV) though edge-preserving can introduce oscillations known as staircasing in homogeneous regions. An adaptive combination via~\eqref{E:ourm} balances the smoothing along with edge preservation.}\label{fig:step}
\end{figure}
\begin{figure}
\centering
	\includegraphics[width=11cm, height=8cm]{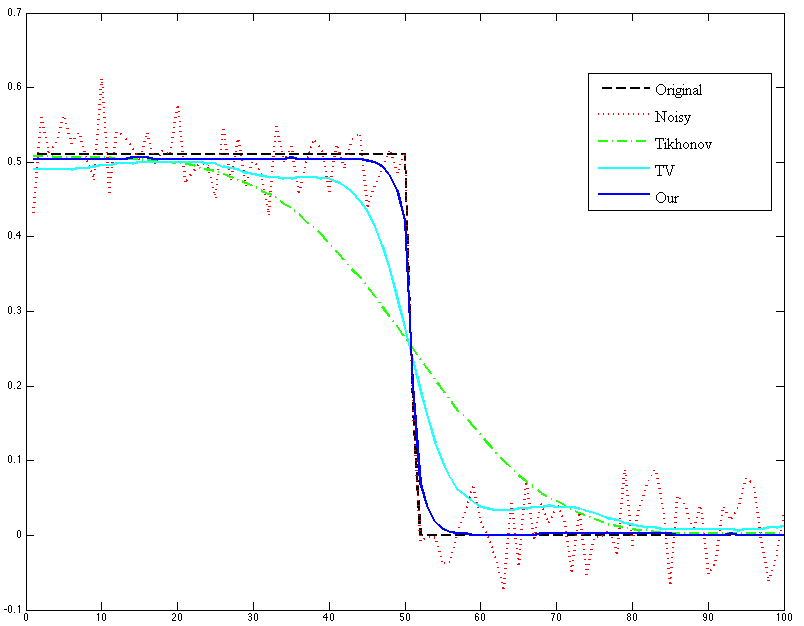}
\caption{One dimensional signal (line) taken across the middle of synthetic $Step$ image in Figure~\ref{fig:step}. The proposed adaptive scheme provides smoothing with edge preservation when compared with Tikhonov (over-smoothing) and TV (staircasing) regularization approaches.}\label{fig:oned}
\end{figure}

\begin{figure}
\centering
    \subfigure[Original \& noisy signals]{\includegraphics[width=2.75in]{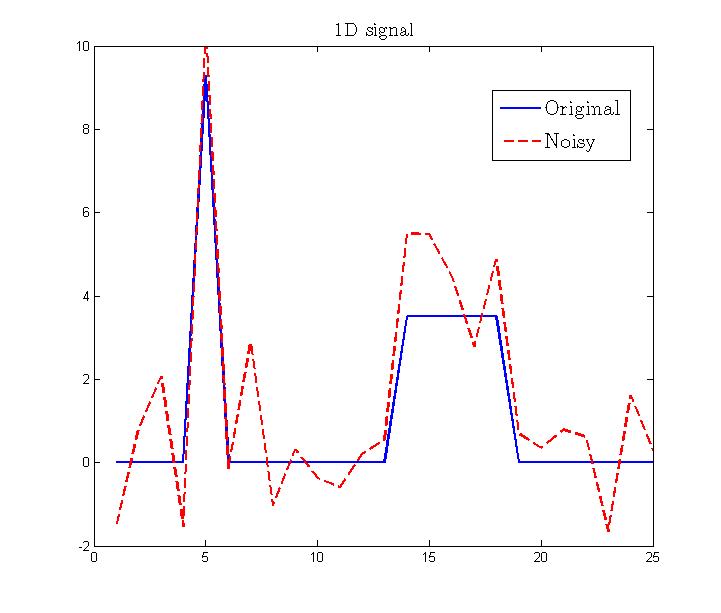}}
    \subfigure[Huber Restorations]{\includegraphics[width=2.75in]{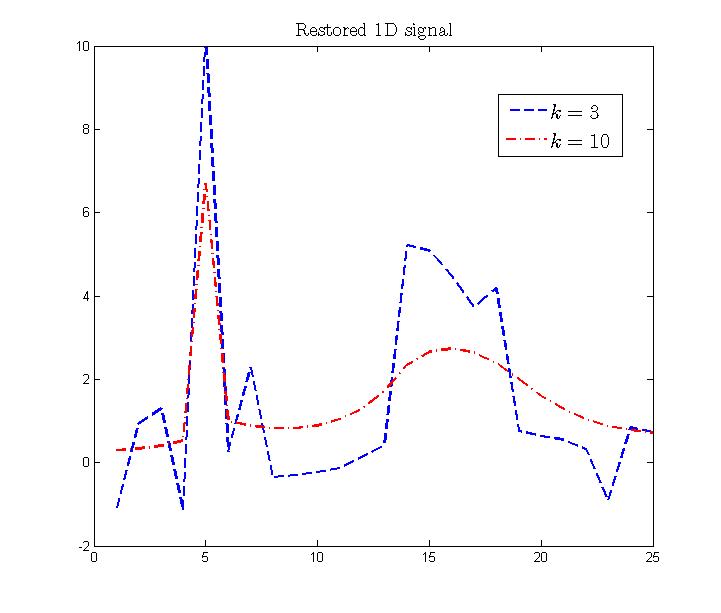}}
    \caption{Dependence on the parameter $k$ for the Huber function $\varphi_H$ give in~(\ref{E:huber}).
    (a) Original signal with an impulse edge at $5$ and a step edge in the range $[14-18]$
     and additive Gaussian noise added (unit variance)
    (b) Restoration using the variational minimization~(\ref{E:regl2l1}) with Huber $\varphi_H$
    in~\eqref{E:huber} with $k=3$ and $k=10$.}\label{I:huber}
\end{figure}
\begin{figure}
\centering
    \subfigure[Original signals]{\includegraphics[width=5.25cm]{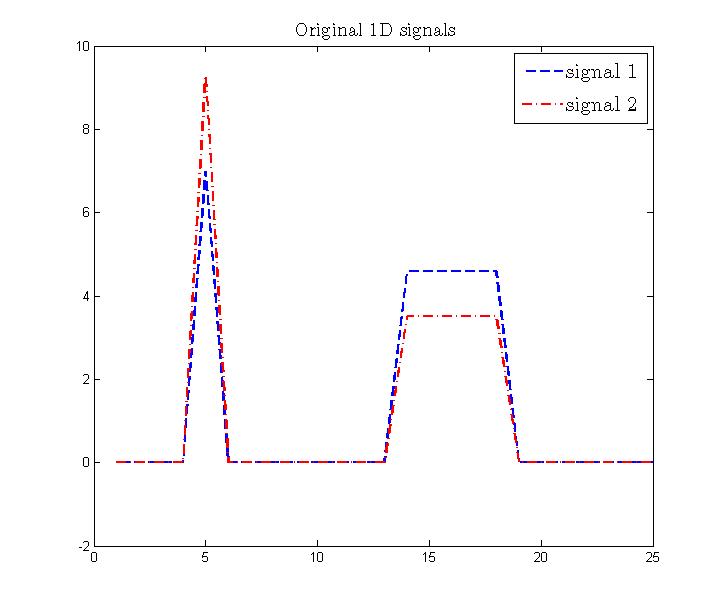}}
    \subfigure[Noisy signals]{\includegraphics[width=5.25cm]{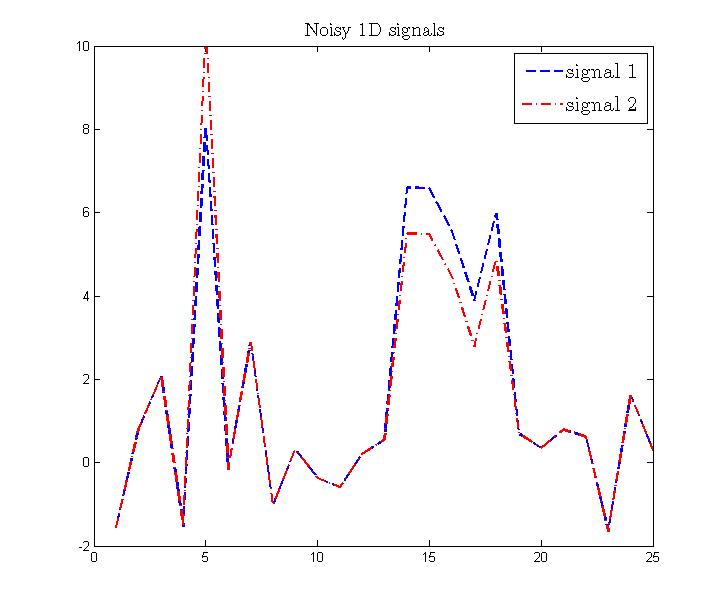}}
    \subfigure[Tukey Restorations]{\includegraphics[width=5.25cm]{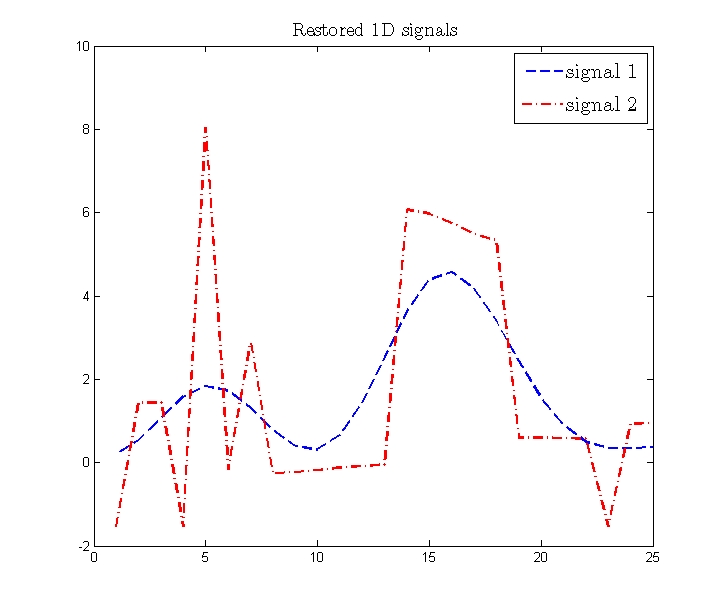}}
    \caption{Instability using Tukey function given in $\varphi_T$~(\ref{E:tukey}).
    (a) Two perturbed signals with different magnitudes 
    (b) Additive Gaussian noise added (unit variance) signals
    (c) Restoration using the minimization~\eqref{E:regl2l1} with Tukey $\varphi_T$
    in~\eqref{E:tukey}.}\label{I:tukey}
\end{figure}

\begin{itemize}
\item The Huber function $\varphi_{H}$~(\ref{E:huber}) is
convex and has a linear response to noisy pixels (outliers) and is strongly
depends on the parameter $k$ for that. Figure~\ref{I:huber} shows
how the dependence on $k$ affects the final restoration strongly on a
1-D noisy signal ($\sigma_n=1$) with two type of discontinuities given in Fig.~\ref{I:huber}(a). If $k$ is
smaller ($k=3$) much of the noise remains and there is no
smoothing, whereas if $k$ is bigger ($k=10$) then smoothing occurs
indiscriminately (Fig.~\ref{I:huber}(b)) and edges are blurred
like the quadratic regularization (equivalent to Gaussian filtering) case. From this we can conclude that
setting a small value for the threshold $k$ captures edges as well
as outliers corresponding to noise. Since we do not a priori know
when and where $\abs{\nabla u}$ jumps (edges) occur and the input image $u_0$ is
corrupted with additive noise there is a need to include an image
adaptive measurement for choosing $k$.

\item On the other hand the Tukey function $\varphi_{T}$~(\ref{E:tukey}) is
non-convex and gives constant response to outliers
(Fig.~\ref{I:mphis})(a), this can be a drawback in a scenario where
the edges and outliers have same high frequency content. To illustrate we consider the
same 1-D signal but perturb slightly to obtain another 1-D signal copy, see Fig.~\ref{I:tukey}(a). The two original signals are of same
type but of different amplitude. After adding additive Gaussian
noise of strength $\sigma_n=1$ to both signals
(Fig.~\ref{I:tukey}(b)) we use Tukey function~(\ref{E:tukey}) based 
minimization scheme~\eqref{E:regl2l1} and obtain the results Fig.~\ref{I:tukey}(c). This
shows that a even slight perturbation of the input signal can
produce a very different output due to instability associated with the non-convexity nature of the regularization function.
\end{itemize}
Motivated by the above arguments and to avoid both the over-under smoothing, and local minima issues, in this paper we use the following regularization function~\cite{PSd10},
\begin{eqnarray}\label{E:ourm}
\varphi_{S}(s)=\begin{cases}
                 a s^2  & \text{if} \abs{s} < k,\\
                 bs^2 + c\abs{s} &\text{if}\abs{s}>k,
            \end{cases}
\end{eqnarray}
where the free parameters $1\gg b>0$ is chosen so as to make the
function lie between quadratic case of Tikhonov and Huber's min-max
function, see Fig.~\ref{I:mphis}(a). This also makes the function to
be in between both $\varphi_{H}$ and $\varphi_{T}$ and strictly convex.
Thus the energy minimization of $E$ in~\eqref{E:regl2l1} is well
posed. For completeness we outline the theorem
here. We denote the the set of all bounded variation functions~\cite{GI84}
from $\Omega\to\mathbb{R}^m$ by
$BV(\Omega;\mathbb{R}^m)$ where $\Omega$ is the image
domain, usually a rectangle in $\mathbb{R}^2$.
\begin{theorem}[Well-posedness]\label{T:exist} 
Let $u_0\in BV(\Omega;\mathbb{R}^m)$ be the initial image. If the regularization function
$\varphi(\cdot)$ is strictly convex then, the energy minimization problem $E(u)$ in~\eqref{E:regl2l1} is well posed in $BV(\Omega;\mathbb{R}^m)$. Moreover, the maximum and minimum principle holds true.
\end{theorem}
\begin{proof}
From~\eqref{E:regl2l1} the first term $(u-I)^2$ is strictly convex
in $u$. Thus if $\varphi$ is also strictly convex then the
well-posedness and maximum - minimum principle follows from~\cite{PSd10}.
\end{proof}
\begin{remark} Note that if
$b\to -1$ in~(\ref{E:ourm}) we approach the Tukey's bisquare
$\phi_{T}$ function continuously but we lose the convexity, see
Fig.~\ref{I:mphis}(a). Hence we stick to $0<b<1$ and use an
adaptive selection of the threshold parameter $k$, see Section~\ref{ssec:param}. 
\end{remark}

Further, to reduce the dependence on the
threshold $k$ we use the following adaptive edge indicator function,
\begin{eqnarray}\label{E:alpha}
	W(x) = \frac{1}{1+K\abs{G_\rho\star\nabla u_0}^2},
\end{eqnarray}
where $K>0$ and $G_\rho$ is the Gaussian kernel with width $\rho>0$, $G_{\rho}=(2\pi\sigma)^{-1}exp(-(|\mathbf{x}|^{2}/2\rho))$ and $\star$ is the convolution operation. Theorem~\ref{T:exist} guarantees that the regularization function $\phi_S$ in~\eqref{E:ourm} with the continuous variational
minimization problem~\eqref{E:regl2l1} is well-posed in
the sense of Hadamard. Note that the data fidelity or the lagrangian parameter $\lambda$  in~\eqref{E:regl2l1} can be made adaptive so that when we use an iterative scheme as in Section~\ref{sec:conv} it is made smaller as the iteration increases. This helps in reducing the regularization as the noise level decreases. An adaptive way to select $\lambda$  in the numerical simulations is given in Section~\ref{sec:exper}. As we will see in denoising examples, this makes our scheme to adjust according to the image information at the current iteration and gives better restoration results overall. If the parameter $\lambda$ is data adaptive, i.e.,
$\lambda=\lambda(u, \nabla u)$ (see Eqn.~\eqref{E:regl2l1}) then the
above theorem holds true if $\lambda\in C^{\infty}(\Omega)$  and continuous, in our case it
is true, see Eqn~\eqref{E:lambda} below.
\section{A convergent finite difference scheme}\label{sec:conv}
\subsection{Discretized functional}\label{ssec:}

The digital image has a natural rectangular grid and without loss of generality we assume that the image $u:\Omega\subset\mathbb{R}^2\to \mathbb{R}$ has size $N\times N$. Then, the domain $\bar\Omega$ is divided into $N^2$ subdomains of side length $h$. We let the vertices $\{v_{i,j}: 1\leq i,j\leq N\}$ so that the $(i,j)^{th}$ square subdomains are $\Omega_{i,j}=v_{i,j} + [-h/2,h/2]^2$. 
Then we use the following finite difference approximations for the gradients,
\begin{eqnarray}\label{E:dgrad}
\nabla^x_{+} u_{ij}= \begin{cases}
				0	&	u_{1j} =0,~~u_{Nj} = 0\\
				\frac{u_{i+1,j} - u_{ij}}{h} & i,j = 1,\ldots,N-1,
				\end{cases} \quad\quad\quad
\nabla^x_{-} u_{ij}= \begin{cases}
				0	&	u_{1j} =0,~~u_{Nj} = 0\\
				\frac{u_{ij} - u_{i-1,j}}{h} & i,j = 1,\ldots,N-1.
				\end{cases} 
\end{eqnarray}
and similarly for the $y$-direction gradients $\nabla^y_{+}$, $\nabla^y_{-}$, to obtain the forward and backward discrete gradients $\nabla_{+} = (\nabla^x_{+},\nabla^y_{+} )$, and $\nabla_{-} = (\nabla^x_{-},\nabla^y_{-} )$ respectively.
Then the discretized functional over $\mathbb{R}^{N\times N}$ is written as,
\begin{eqnarray}\label{E:discf}
E_h (u) = \sum_{1\leq i, j \leq N} \phi_h (W_{ij}(\nabla u)_{i,j})+ \frac{h^2\lambda}{2}\sum_{1\leq i,j\leq N} (u_{i,j}- (D_hu_0)_{i,j})^2
\end{eqnarray}
where $D_h$ is the discrete operator applied to the input image $u_0$.
The discrete regularizer in the above equations is,
\begin{eqnarray}\label{E:discour}
\phi_h (W_{ij}(\nabla u)_{i,j}) = \frac{W_{ij}h^2}{2}\times\begin{cases}
                 a \left(\abs{\nabla_+ u_{i,j}}^2 + \abs{\nabla_{-}u_{i,j}}^2\right)  & \text{if} \abs{s} < k,\\
                 b \left(\abs{\nabla_+ u_{i,j}}^2 + \abs{\nabla_{-}u_{i,j}}^2 \right) + c\left( \abs{\nabla_+ u_{i,j}} + \abs{\nabla_{-}u_{i,j}} \right) &\text{if}\abs{s}>k,
            \end{cases}
\end{eqnarray}
with $W_{ij}$ the discrete version of the edge indicator function~\eqref{E:alpha} using the discrete gradient and the discrete window based Gaussian function.
\subsection{Split Bregman method}\label{ssec:split}

We recall the split Bregman method to solve the discrete energy functional in Eqn.~\eqref{E:discf}. We sketch the main parts of the algorithm here and we refer
to~\cite{GB10} and~\cite{GO09} for the general treatment on Split
Bregman approach. This is a very fast scheme, faster than other
numerical schemes reported in the literature, as we will see for
example in image denoising tasks, Section~\ref{sec:exper}. An auxiliary variable
$\vec{d}\leftarrow\nabla u$ is introduced in the
model with a quadratic $L^2$ penalty function. That is to solve the TV minimization,
\begin{eqnarray}\label{E:tv}
\min_u TV(u) =  \int_{\Omega}\abs{\nabla u}\,dx,
\end{eqnarray}
we consider the
following unconstrained minimization problem,
\begin{eqnarray}\label{E:sbun}
\min_{u,\vec{d}} \left\{\vert\vec{d}\vert +
\frac{\lambda}{2}\norm{\vec{d}-\nabla u}^2_{L^{2}(\Omega)}\right\}.
\end{eqnarray}
The above problem is solved by using an alternating minimization
scheme, which includes the addition of a vector $\vec{e}$, inside
the quadratic functional. That is, the algorithm reduces to the
following sequence of unconstrained problems, 
\begin{align}
(u^{t+1},\vec{d}^{t+1}) &= arg\min_{0\leq u\leq
1,\;\vec{d}}~\vert\vec{d}\vert +
\frac{\lambda}{2}\norm{\vec{d}-\nabla u-\vec{b}^t}^2_{L^{2}(\Omega)}\label{E:sb1}\\
\vec{e}^{t+1} & = \vec{e}^t + \nabla u^t -\vec{d}^t\label{E:sb2}
\end{align}
First a minimization with respect to $u$ is performed using a
Gauss-Seidel  method.  Next a minimization with respect to
$\vec{d}$ is done using a shrinkage method. Finally, the vector $\vec{e}$
is updated using~\eqref{E:sb2}. The following steps summarize the algorithm,
\begin{enumerate}
\item Initialize $d^0,e^0\in (L^2(\Omega))^n$

\item For $t\geq 1$
    \begin{enumerate}
        \item $(\mu I - \lambda\Delta) u^{t+1} = \mu u_0 - \nabla^{T}(d^t - e^t)$
        \item Compute \[d^{k+1} = shrink\left(\nabla u^t + e^t,\frac{1}{\lambda}\right)\]
    \end{enumerate}
\item $e^{t+1} = e^t+ \nabla u^{t+1} - d^{t+1}$
\end{enumerate}
The shrinkage operation is given by,
\[shrink(x,\gamma) =
\frac{x}{\abs{x}}*\max{(\abs{x}-\gamma,0)}.\] 
It can be shown that this algorithm converges very quickly even when an approximate
solution is used in Eqn.~\eqref{E:sb1}. The split Bregman algorithm for solving our functional~\eqref{E:discf} can similarly be derived. Note that in our case the shrinkage becomes
\begin{eqnarray}\label{E:ourshrink}
d^{k+1} = shrink\left(\nabla u^t + e^t,\frac{W}{\lambda}\right),
\end{eqnarray}
where $W$ is the adaptive edge indicator function given in Eqn.~\eqref{E:alpha}.
\subsection{Convergence}\label{ssec:con}

The digital image $u\in\mathbb{R}^{N\times N}$ is interpolated using continuous piecewise linear functions on $\Omega$,
\[\mathcal{P}_hU(x) =  \sum_{1\leq i,j \leq N} U_{i,j}\ell_{i,j}(x)\]
with $\ell_{i,j}:\Omega\to\mathbb{R}$ and $\ell_{i,j}(v_{i,j}) = 1$, $\ell_{i,j}(v) = 0$, $\omega\in\{v_{i,j}\}^c$. Similarly we define piecewise constant extension $\mathcal{C}_hU(x) = U_{i,j}$ for $x\in int(\Omega_{i,j})$, and the sampling operator 
\[\mathcal{Q}_h U(x)= \frac{1}{\abs{\Omega_{i,j}}}\int_{\Omega_{i,j}} U(y)\,dy,~~\text{for}~x\in int(\Omega_{i,j}).\] 
To prove the convergence of the interpolated function to the continuous solution we first introduce some basic notations. In what follows we use the standard notations on Lebesgue $L^p(\Omega)$ ($1\leq p\leq \infty$) and functions of bounded variation $BV(\Omega)$ spaces. We define the translation of a set and and a function with vector $\tau\in\mathbb{R}^2$ as $T^\tau\Omega = \{x+\tau : x\in\Omega\}$, $T^\tau \phi(x) = \phi(x+\tau)$ for $x\in T^{-\tau}\Omega$ respectively. Let us recall the definition of $p$-modulus of continuity of order $t >0$ for a function $\phi\in L^p(\Omega)$, $\omega(\phi,t)_p = \sup_{\abs{\tau}\leq t} \norm{T^\tau\phi - \phi}_{L^p(\Omega\cap T^{-\tau}\Omega)}$. Note that the modulus of continuity gives a quantitative account of the continuity property of $L^p(\Omega)$ functions.

\begin{definition}[Weakly regular functions]
Let $\phi\in L^p(\Omega)$ and $0 < \mathcal{L} \leq 1$. We say $\phi$ is weakly regular ($\mathcal{L}$-Lipschitz) function if it satisfies the condition $\sup_{0<t<1}\frac{\omega(\phi,t)}{t^{\mathcal{L}}}<\infty$.
\end{definition}
The main convergence theorem is stated as follows.
\begin{theorem}[Convergence]\label{T:main}
Let $u_0\in L^\infty(\Omega)$, weakly regular ($\mathcal{L}$-Lipschitz, $\mathcal{L}\in (0,1]$) and $D_hU_0$ be the discretization with respect to a uniform quadrangulation $Q_h$. Let $U$ be the minimizer of the discretized functional over $\mathbb{R}^{N\times N}$,
\[E_h (u) = \sum_{1\leq i, j \leq N} \phi_h (W_{ij}(\nabla u_{i,j}))+ \frac{h^2\lambda}{2}\sum_{1\leq i,j\leq N} (u_{i,j}- (D_hu_0)_{i,j})^2\]
which is obtained using the split Bregman scheme in Section~\ref{ssec:split}, and $u$ be the minimizer of the continuous functional~\eqref{E:regl2l1}. Then,
\begin{enumerate}
\item[(i)] The interpolated solution of the discrete model converges to the continuous solution,
\[\norm{\mathcal{P}_h(U) - u}_{L^2(\Omega)}\to 0~~\text{as}~~h\to 0.\] 

\item[(ii)] $E_h(\mathcal{P}_h(U))$ converges to $E(u)$ as $h\to 0$.
\end{enumerate}
\end{theorem}
We derive some preliminary results required for proving the main theorem. We use a generic constant $C$ which can change in line to line.
\begin{lemma}[Bounds on solutions]\label{L:bound}

\begin{itemize}
\item[(1)] Continuous:
Let $\tilde u\in BV(\Omega;\mathbb{R}^m)$ be a solution of the energy minimization~\eqref{E:regl2l1} with the adaptive regularization function~\eqref{E:ourm}. If  $u^*\in BV(\Omega;\mathbb{R}^m)$, then
\begin{eqnarray}\label{E:solns}
\norm{\tilde u - u^*}_2^2 \leq \frac{2}{\lambda} \abs{E(\tilde u) - E(u^*)}
\end{eqnarray}
\item[(2)] Discrete: Let $\tilde U\in \mathbb{R}^{N\times N}$ be the minimizer of the discretized functional $E_h$ in \eqref{E:discf}. Then
\begin{eqnarray}\label{E:dsolns}
E(\mathcal{P}_h(\tilde U)) - E_h(\tilde U) \leq  \frac{\lambda}{2} C\omega(u_0,h)_2 [C\omega(u_0,h)_2 + 8\norm{u_0}_2]
\end{eqnarray}
\end{itemize}
\end{lemma}

\begin{proof}
(1) The inequality follows from the fact that for the adaptive regularization~\eqref{E:ourm} based energy minimization functional $E$in Eqn.~\eqref{E:regl2l1} is $L^2$-subdifferentiable. 

(2) We first note that 
\begin{align*}
\norm{\mathcal{P}_h\mathcal{Q}_hu_0 - u_0}_2
&\leq C\omega(u_0,h)_2\quad\text{and} \quad \norm{\mathcal{P}_h(\tilde U - \mathcal{Q}_hu_0)}^2_2\leq 4\norm{u_0}_2.
\end{align*}
Then the inequality~\eqref{E:dsolns} follows from,
\begin{align*}
\frac{2}{\lambda} (E(\mathcal{P}_h(\tilde U)) - E_h(\tilde U)) \leq \norm{\mathcal{P}_h\mathcal{Q}_hu_0 - u_0}_2 \left\{\norm{\mathcal{P}_h\mathcal{Q}_hu_0 - u_0}_2 + 2 \norm{\mathcal{P}_h(\tilde U - \mathcal{Q}_hu_0)}^2_2\right\}.
\end{align*}
\end{proof}
\begin{lemma}[Convolution bound]\label{L:conv}
Let $\tilde U\in \mathbb{R}^{N\times N}$ be the minimizer of the discretized functional $E_h$ in \eqref{E:discf}. Let $u_\epsilon = G_\epsilon\star u$ be the mollified extension of the image function $u\in BV(\Omega)$ to $u\in BV(\mathbb{R}^2)$. Then
\begin{eqnarray*}
E_h(\tilde U) - E(u_\epsilon) \leq C\norm{u_0}_\infty^2+ \mathcal{O}(h/\epsilon^2).
\end{eqnarray*}
\end{lemma}
\begin{proof}
First note that
\begin{align*}
E_h(\tilde U) &\leq E_h(u_\epsilon)\\
& \leq \int_{\Omega} \abs{\nabla \mathcal{P}_hu_\epsilon}^2\,dx + \frac{\lambda}{2} \sum_{1\leq i, j \leq N} h^2\abs{(u_\epsilon)_{i,j} - (\mathcal{Q}_hu_0)_{i,j}}^2
\end{align*}
and
\[\norm{\mathcal{P}_hu_\epsilon - u_\epsilon}_{W^{1,2}} \leq C h \sum_{\abs{\alpha} =2} \norm{D^\alpha u_\epsilon}_2 \leq Ch/\epsilon^2\]
Then the inequality follows from,
\[\sum_{1\leq i, j \leq N} h^2\abs{\mathcal{Q}_h(u\epsilon - u_0)_{i,j}}^2 \leq \norm{u_\epsilon - u_0}^2_2 + C\norm{u_0}^2_\infty.\]
and
\[\sum_{1\leq i, j \leq N} h^2\abs{(u_\epsilon)_{i,j} - (\mathcal{Q}_hu_0)_{i,j}}^2 \leq \sum_{1\leq i, j \leq N} h^2 \abs{(\mathcal{Q}_hu_\epsilon - \mathcal{Q}_hu_0)_{i,j}}^2 + C\mathcal{O}(h/\epsilon^2).\]
\end{proof}
\noindent\textbf{Proof of Theorem 2:}\\
Let $\epsilon>0$  and $h \leq 1$. From Eqn.~\eqref{E:solns},
\begin{align*}
\norm{\mathcal{P}_h(U) - u}^2_2 &\leq \frac{2}{\lambda} \{E(\mathcal{P}_hU) - E(u)\}\\
&\leq \frac{2}{\lambda} \{(E(\mathcal{P}_hU) - E_h(U)) + (E_h(U) - E(u))\}
\end{align*}
Using Lemma~\ref{L:bound} and Lemma~\ref{L:conv} respectively for the two difference terms we obtain,
\begin{eqnarray}
\norm{\mathcal{P}_h(U) - u}^2_2 \leq \omega(u_0,h)_2 \{\omega(u_0,h)_2+ C\norm{u_0}_2\}  + \frac{32h}{\lambda}\norm{u_0}_\infty^2 + \frac{2Ch}{\lambda\epsilon^2} + \frac{2}{\lambda} \{E(u_\epsilon) - E(u)\}.
\end{eqnarray}
Let $\epsilon=h^{1/(2\mathcal{L}+1)}$ and since $u_0$ is weakly regular $\omega(u_0,h)_2\leq \mathcal{O}(h^\mathcal{L})$, the above inequality becomes
\[\norm{\mathcal{P}_h(U) - u}^2_2 \leq \frac{2}{\lambda} \{E(u_\epsilon) - E(u)\} + C h^{\mathcal{L}/(\mathcal{L}+1)}\]
Since $E(u_\epsilon) - E(u)\to 0$ as $\epsilon\to 0$ we have the result.
\hfill$\qed$
\section{Experimental results and discussion}\label{sec:exper}

\subsection{Parameters}\label{ssec:param}

\begin{figure}
\centering
	\subfigure[Noisefree]{\includegraphics[width=3.75cm]{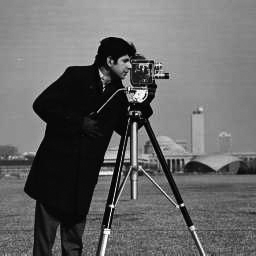}}
	\subfigure[Noisy, $\sigma_n=20$]{\includegraphics[width=3.75cm]{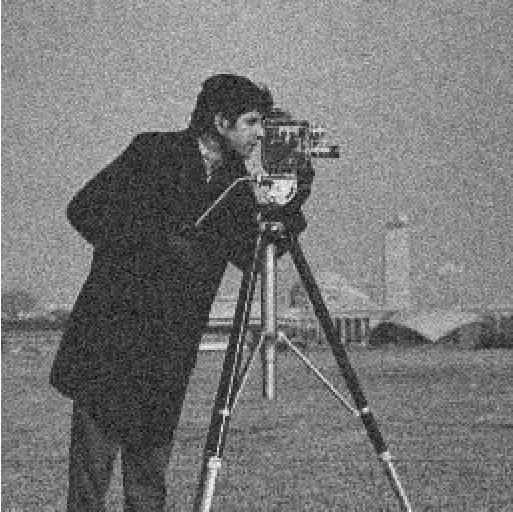}}
	\subfigure[$1-\abs{\nabla u_0}$]{\includegraphics[width=3.75cm]{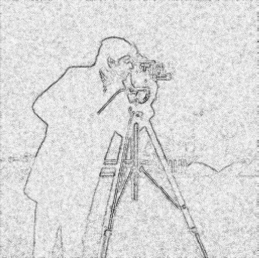}}
	\subfigure[$\lambda^{100}$]{\includegraphics[width=3.75cm]{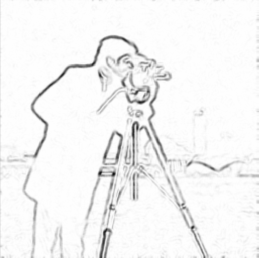}}
	\caption{Original $Cameraman$ gray scale test image of size $256 \times 256$ used in our experiments and its edge map computed with gradients.
	(a) Noise free image
	(b) Gaussian noise added image, $\sigma_n=20$
	(c) Gradient image, inverted ($1-\abs{\nabla u_0}$) for better visualization
	(d) Adaptive $\lambda$ from Eqn.~\eqref{E:lambda} at iteration $t=100$. Notice that the edges are preserved whereas the noise is removed in homogeneous regions as the iterations are increased.}\label{fig:camoriginal}
\end{figure}
We set the step size $h = \delta t = 0.20$, $a=1$,
and parameters in our regularization function in~\eqref{E:ourm} to $b = 0.05$, $\rho=2$, and the thresholding parameter $k$ is determined using the mean absolute deviation (MAD) from robust statistics~\cite{RE83},
\begin{align*}
k & = 1.4826\, \times \text{MAD}(\nabla u)\\
  & = 1.4826\,\times \text{median}_u [\abs{\nabla u - \text{median}(\abs{\nabla u})}]
\end{align*}
where the constant is derived from the fact that the MAD of a
zero-mean normal distribution with unit variance is
$0.6745=1/1.4826$. For the discrete functional~\eqref{E:discf}, the parameter $k$ is
computed using the gradient magnitude $\abs{\nabla u}$ for which we used the same finite difference approximations introduced before, see Eqns.~\eqref{E:dgrad}.  All the test images are normalized to the range $[0, 1]$.

We further introduce an iteration and pixel adaptive
$\lambda^{(t)}_{i,j}$ using the gradient information at iteration
$(t-1)$ via
\begin{eqnarray}\label{E:lambda}
\lambda^{(t)}_{i,j}: = \frac{1}{\epsilon^2+\sqrt{(u^{(t-1)}_{i+1,j} -
u^{(t-1)}_{i,j})^2 + (u^{(t-1)}_{i,j+1} -
u^{(t-1)}_{i,j})^2 }}
\end{eqnarray}
where $\epsilon^2 =10^{-6}$ is added to avoid numerical
instabilities. Note that $\lambda^{(t)}_{i,j}\in [0,1]$ reduces the influence
of the regularization term at edges and makes the scheme an image adaptive method.
This also reduces the dependence on the threshold $k$ to decide
upon the outliers part (compare this with Huber's minmax function~\eqref{E:huber}
and Fig.~\ref{I:huber}). Since Theorem~\ref{T:exist} implies
stability we are guaranteed of a good reconstruction even if the
input is perturbed significantly (compare this with Tukey bisquare
function~\eqref{E:tukey} and Fig.~\ref{I:tukey}).  Fig.~\ref{I:oned} we consider the same 1-D signal shown earlier in Fig.~\ref{I:huber}(a). The restoration result exhibits strong smoothing property of our adaptive regularization function with edge preservation. Exact locations of the true discontinuities are preserved and noise is completely removed in homogenous regions. 

Figure~\ref{fig:camoriginal} show the $Cameraman$ gray-scale $256\times 256$ size image used in our later comparison results. We add Gaussian white noise of standard deviation $\sigma_n=20$ and mean zero\footnote{Using MATLAB command \texttt{imnoise($u_0$,'gaussian',0,$\sigma_n$)}.}.  Figure~\ref{fig:camoriginal} (b) \& (c) shows the gradient image (Computed using the formulae~\eqref{E:dgrad}) from the initial noisy image $\abs{\nabla u_0}$ and adaptive $\lambda$ parameter computed using Eqn.~\eqref{E:lambda} at iteration $100$ showing the improvement in the edge map.

\subsection{Restoration results}\label{ssec:rest}
\begin{figure}
\centering
    \includegraphics[width=3.5in]{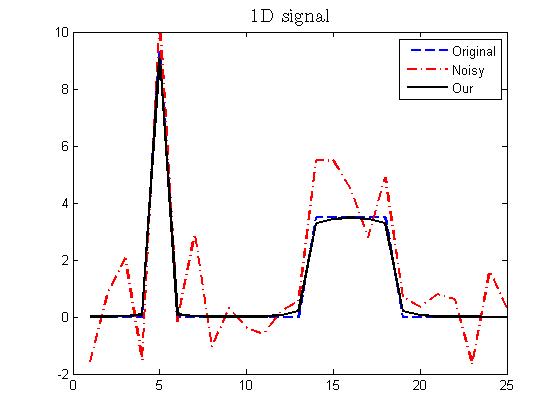}
    \caption{Restoration of a 1-D signal by our scheme: ($---$) Original signal ($-\cdot-$) noisy signal and
    (-----) dark line is the restored curve using our adaptive minimization scheme~\eqref{E:regl2l1} with function~\eqref{E:ourm}. Compare this with the corresponding results for Huber and Tukey functions in Figure~\ref{I:huber}(b) and Figure~\ref{I:tukey}(c) respectively.}\label{I:oned}
\end{figure}
\begin{figure}
\centering
    \subfigure[Color $Movie$ scene and $Kid$,  $Goat$ gray scale photo]
     {\includegraphics[width=6cm, height=4cm]{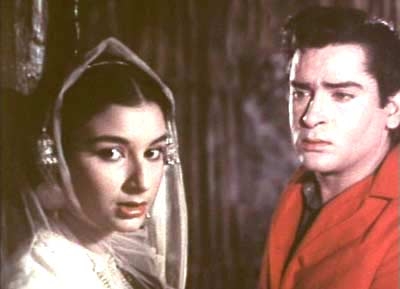} \includegraphics[width=3.75cm,  height=4cm]{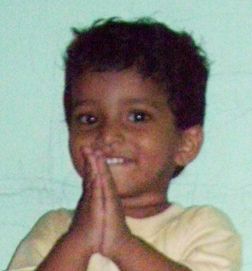}\includegraphics[width=5cm,  height=4cm]{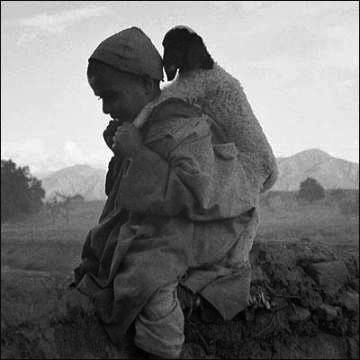}}\\
     \subfigure[Smoothed images at iteration $t=100$]{\includegraphics[width=6cm, height=4cm]{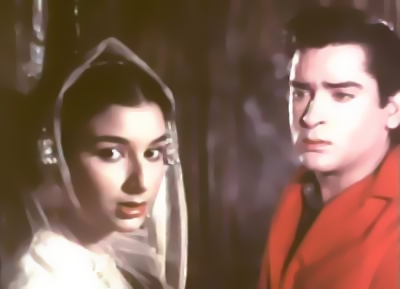}\includegraphics[width=3.75cm,  height=4cm]{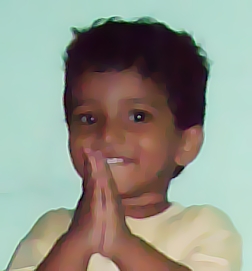} \includegraphics[width=5cm,  height=4cm]{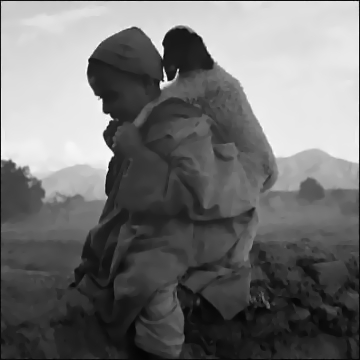}}
    \caption{Restoration by our adaptive regularization scheme on some real images with unknown noise strength. 
    (Top row) Original images 
    (Bottom row) Our adaptive regularization scheme results.}\label{I:kash}
\end{figure}
\begin{figure}
\centering
    \subfigure[Original]{\includegraphics[width=3.75cm]{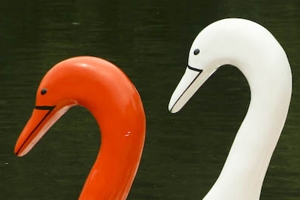}}
    \subfigure[Input]{\includegraphics[width=3.75cm]{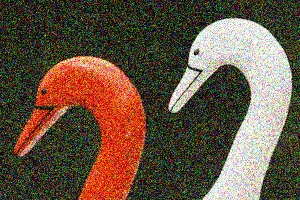}}
    \subfigure[Result]{\includegraphics[width=3.75cm]{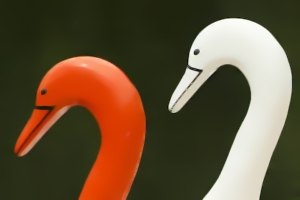}}
    \subfigure[Edges]{\includegraphics[width=3.75cm]{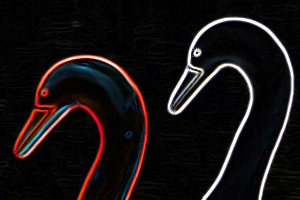}}
    \caption{$Ducks$ color image $300\times 200\times 3$ restoration result. 
    (a) Original RGB image
    (b) Noisy image, $\sigma_n = 20$, $PSNR = 12.56dB$ 
    (c) Restored by our method, $PSNR = 23.15dB$
    (d) Edges computed in all three channels (RGB) using the Canny edge detector.}\label{I:ducks}
\end{figure}

In Fig.~\ref{I:kash} we restore three real images, original color $Movie$ still (film grain noise, medium granularity), a $Kid$ image taken by a mobile camera picture ($2$ mega-pixels, image contains unknown amount of shot noise), and $Goat$ an old gray-scale photograph (noise type unknown) respectively. Note that for (RGB) color images we use the scheme~\eqref{E:regl2l1} for each of the channels red, green, and blue and combine the final restoration result. The restored results in Fig.~\ref{I:kash} (b) exhibit marked improvements. Note that fine texture details are lost in Fig.~\ref{I:kash} (b) (background wall, goat, hair and shirt), we may need to include further statistical information about textures in our scheme. Apart from this our scheme overall performs well and has strong edge preserving smoothing properties. The strong smoothing nature of our adaptive regularization~\eqref{E:ourm} can be seen in another piecewise smooth image shown in Fig.~\ref{I:ducks} (a). This $Ducks$ color image consists of flat background with strong curved edges and the result in Fig.~\ref{I:ducks} (c) indicates the local smoothing due to Gaussian filtering effect in regions where $\abs{\nabla u} <k$ and edge preserving TV filtering in other areas. Figure~\ref{I:ducks} (d) shows the Canny edge map of computed in all the three color channels\footnote{Using MATLAB command \texttt{edge($u_0$,'canny')}. Note that the Canny edge detector employs non-maximal suppression to avoid small scale edges. The edges are computed for each of Red, Green, Blue channels and the final result is shown by combining them.}. 

\subsection{Comparison results}\label{ssec:comp}

\begin{figure}
\centering
    \subfigure[True image $u$]{\includegraphics[width=3.8cm]{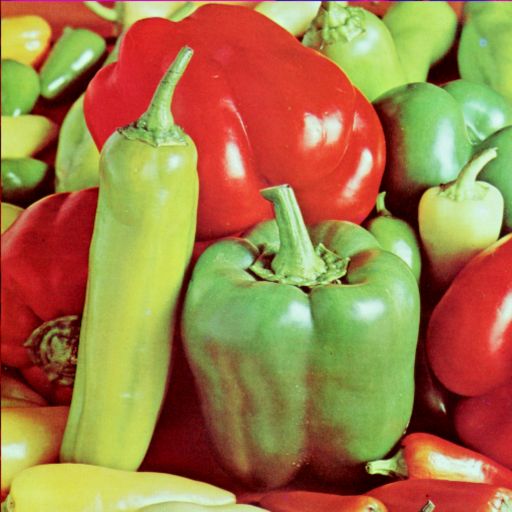}}
    \subfigure[Huber $U^H$]{\includegraphics[width=3.8cm]{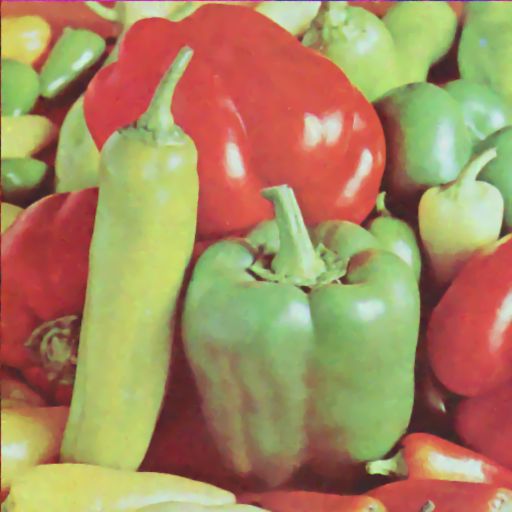}}
    \subfigure[Tukey $U^T$]{\includegraphics[width=3.8cm]{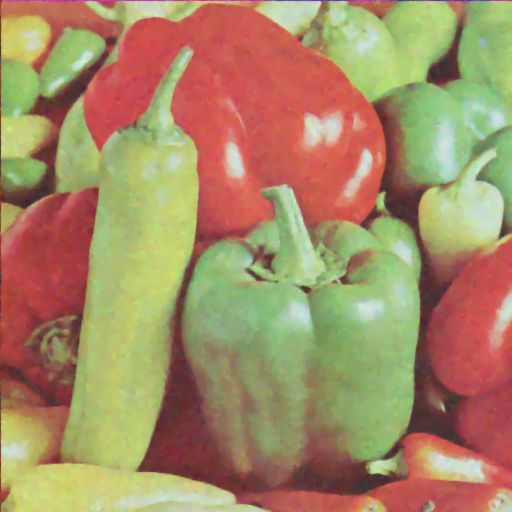}}
    \subfigure[Our $U^{our}$]{\includegraphics[width=3.8cm]{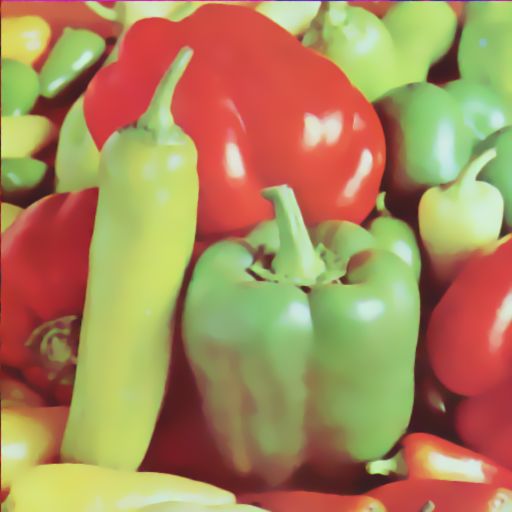}}
    \\
    \subfigure[Noisy image $I$]{\includegraphics[width=3.8cm]{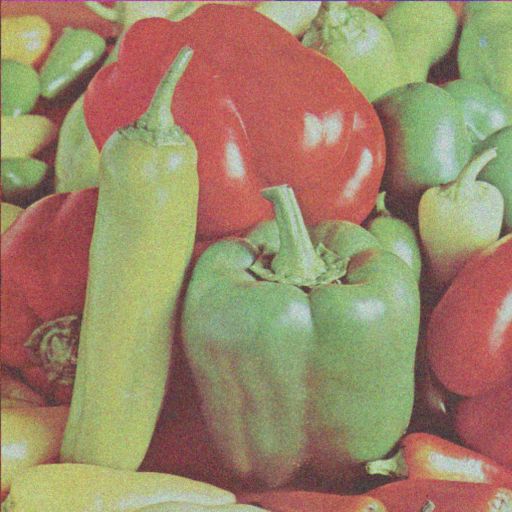}}
    \subfigure[]{\includegraphics[width=3.8cm]{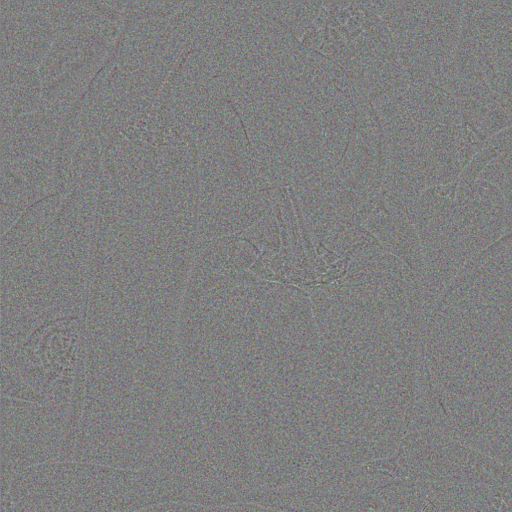}}
    \subfigure[]{\includegraphics[width=3.8cm]{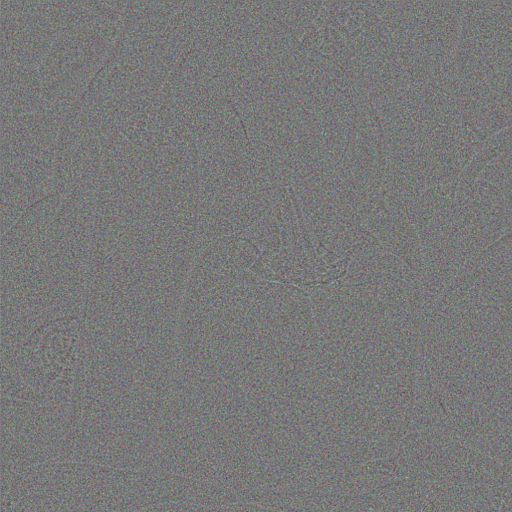}}
    \subfigure[]{\includegraphics[width=3.8cm]{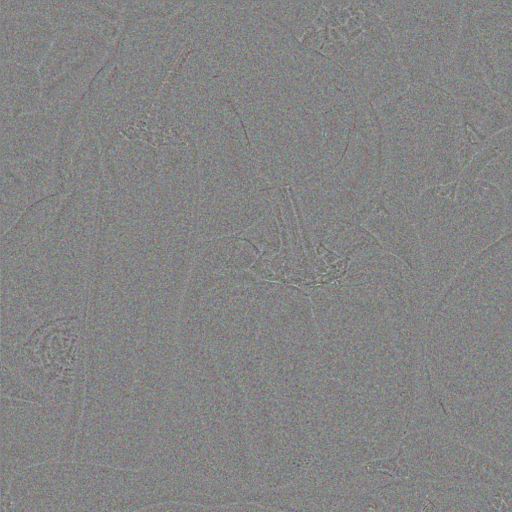}} 
    \subfigure[Noisy contour]{\includegraphics[width=3.8cm, height=3.8cm]{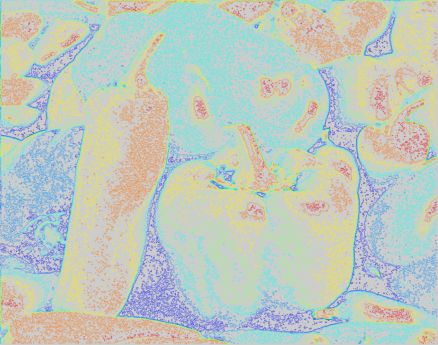}}
    \subfigure[]{\includegraphics[width=3.8cm, height=3.8cm]{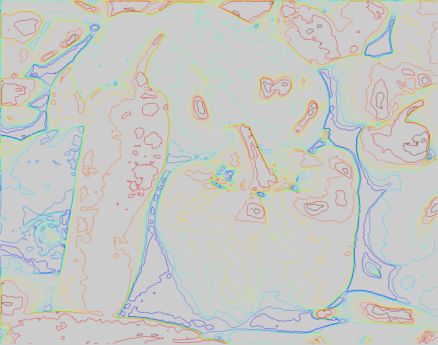}}
    \subfigure[]{\includegraphics[width=3.8cm, height=3.8cm]{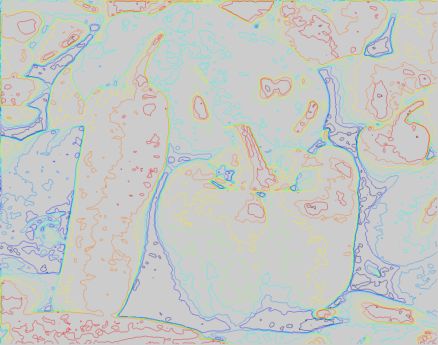}}
    \subfigure[]{\includegraphics[width=3.8cm, height=3.8cm]{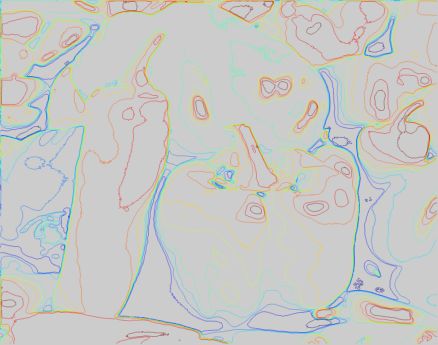}} 
    \caption{Restoration result for a color $Peppers$ color image (size $512\times 512\times 3)$ with our adaptive regularization scheme:
    (a) Original image 
    (b-d) results of Huber, Tukey and our adaptive regularization function based scheme with tolerance $tol=10^{-6}$ respectively
    (e) Gaussian noise ($\sigma_n =20$) corrupted image
    (f-h) Residual noise/method noise image, $\abs{u-U}^2$
    (i-j) Contour map showing the restoration on level lines.}\label{I:pepper}
\end{figure}
\begin{figure}
\centering
    \subfigure[Iteration Vs Mean Error]{\includegraphics[width=3.in, height=2.in]{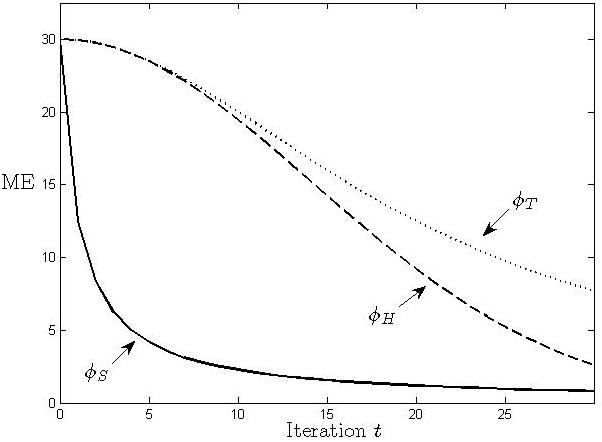}}
    \subfigure[Noise Vs Peak Signal-to-Noise Ratio]{\includegraphics[width=3.in, height=2.in]{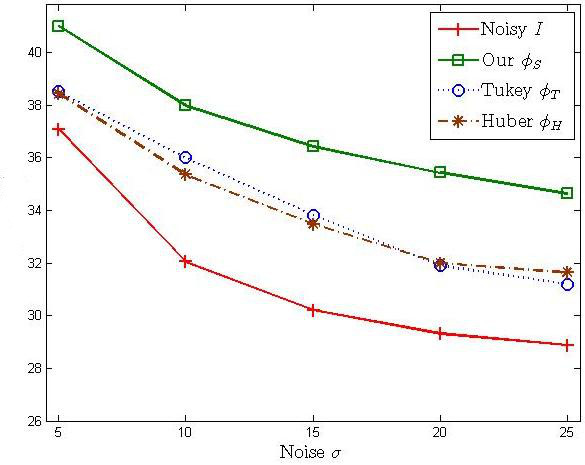}}
    
    \caption{Comparison of our proposed scheme with $\varphi_S$ in~(\ref{E:ourm}) with Huber's
    $\varphi_H$~\eqref{E:huber} and Tukey's $\varphi_{T}$:~\eqref{E:tukey} for the $Peppers$ color image in Figure~\ref{I:pepper}.
    (a) Number of iterations ($t$) Vs Mean error ($ME$)
    (b) Noise level ($\sigma_n$) Vs Peak signal-to-noise ratio ($PSNR$) for different noise levels.}\label{I:me}
\end{figure}

Figure~\ref{I:pepper} we show a comparison of restoration results for the $Peppers$ color image. As can be seen from the method noise and a contour maps our adaptive regularization scheme outperforms other schemes in terms of noise removal and edge preservation. The level lines are smoothed without reducing their edginess and flat regions are preserved without staircasing artifacts. Figure.~\ref{I:me} shows the ME and PSNR comparisons illustrating the versatility of our adaptive scheme~\eqref{E:regl2l1} with the proposed regularization function~\eqref{E:ourm} against other functions. Also note that the ME error curve for our method outperforms Huber and Tukey functions based regularization and quickly converges to a desired solution (usually $t=50$ is sufficient). On the other hand our function~(\ref{E:ourm}) is robust when compared to the other two classical functions as can be seen from the PSNR comparison Fig.~\ref{I:me} (b) as well. The topmost PSNR curve indicates that the scheme proposed in this paper surpasses the other two when the noise level increases $\sigma_n =10\to 25$. Note that $\sigma_n^2 > 400$ is a high level noise and our scheme~\eqref{E:regl2l1} does a good job in distinguishing between outliers correspond to noise and true edges due to the adaptive nature of $\lambda^{(t)}_j$~\eqref{E:lambda}.

We next provide comparison with primal dual hybrid gradient (PDHG)~\cite{ZhuChanPDHG08}, projected averaged gradient (Proj. Grad)~\cite{ZhuProjGrad10}, fast gradient projection (FGP)~\cite{BeckFGP09}, alternating direction method of multipliers (ADMM)~\cite{GlowinskiADMM75}, and split Bregman (Split Breg.) based schemes. The following error metrics are used to compare the convergence and performance of different algorithms for the discrete minimization Eqn.~\eqref{E:discf}.

\begin{itemize}
\item Relative duality gap:
\begin{eqnarray}\label{E:gap}
 \mathcal{R}(u,b) = \frac{E_{Primal}(u) - E_{Dual}(b)}{E_{Dual}(b)},
\end{eqnarray}
where $E_{Primal}$, $E_{Dual}$ represent the primal and dual objective functions respectively. This is used as a stopping criteria for the iterative schemes.

\item Peak Signal-to-Noise (PSNR) ratio,
\begin{eqnarray}\label{E:psnr}
\text{PSNR}=  20*\log10{\left(\frac{255}{\sqrt{\sum_{1\leq i,j \leq N} (u-u_0)^2}}\right)} (dB) 
\end{eqnarray}
The higher the PSNR the better the restoration result.

\item The mean error (ME):
\[ME(u,I) := \frac{1}{MN}\sum_{i} \abs{u_i - I_i}\]
The mean error needs to be small for restored images.
\end{itemize}

\begin{figure}
\centering
 	\subfigure[Tikhonov]{\includegraphics[width=3.25cm, height=3.25cm]{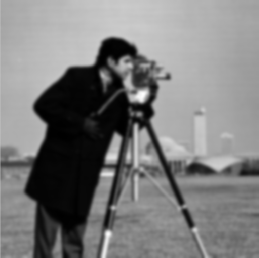}}
 	\subfigure[TV]{\includegraphics[width=3.25cm, height=3.25cm]{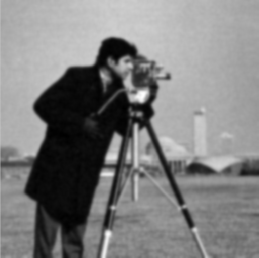}}
 	\subfigure[Our]{\includegraphics[width=3.25cm, height=3.25cm]{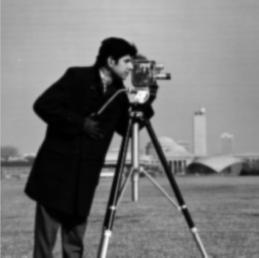}}
	 \subfigure[PSNR (dB) comparison]{\includegraphics[width=4.5cm, height=3.25cm]{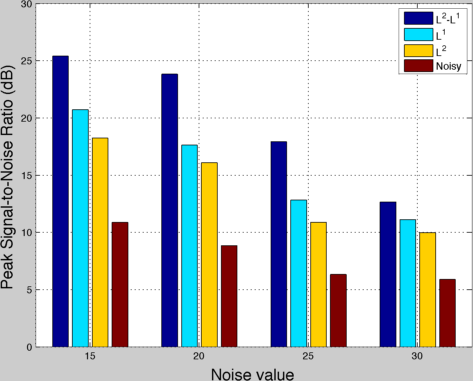}}
	\caption{Image restoration of noisy $Cameraman$ (Figure~\ref{fig:camoriginal} (b)) and PSNR (dB) comparison of results for the (a) Tikhonov ($\varphi(s) = s^2$), (b)  TV ($\varphi(s) = s$) and (c) our regularization function~\eqref{E:ourm} based schemes (d) PSNR comparison shows that the proposed adaptive scheme performs better across different noise levels.}\label{fig:cameraman}
\end{figure}
\begin{table}[ht]
\centering
    \begin{tabular}{l|lll}
    \hline
    Algorithm                    & tol $=10^{-2}$ 	& tol $=10^{-4}$ 	& tol $=10^{-6}$\\
    \hline
    PDHG			     & 14             & 70             & 310\\
    Proj. Grad.                  & 46             & 721            & 14996\\
    FGP                   	    & 24             & 179            & 1264\\
    ADMM                         & 97             & 270            & 569\\
    Split Breg.                  & \textbf{10}             & \textbf{28}            & \textbf{55}\\
    \hline
    \end{tabular}
    \caption{Comparison with primal dual hybrid gradient (PDHG), projected averaged gradient (Proj. Grad), fast gradient 		projection (FGP), alternating direction method of multipliers (ADMM), and split Bregman based scheme. 			Iterations required for denoising of the $Cameraman$ image ($256\times 256$, noise level $\sigma_n=20$) with 			different numerical schemes for the relative duality gap $\mathcal{R}(u,b)\leq tol$.}\label{t:gap}
\end{table}
\begin{table}[ht]
\centering
 \begin{tabular}{cccccc}
 	\hline
 Noise & PDHG		    	        & Proj. Grad. 		   & FGP			       &	ADMM		      & Split Breg.\\
 \hline
  15	   & 21.61 (28s, $100$)  & 21.58 (30s, $85$)  & 20.21 (20s, $70$) & 21.85 (24s, $73$) & \textbf{25.40} (10s, $55$)\\
  20	  &  20.46 (28s, $86$)  & 20.29 (30s, $80$)  & 20.12 (20s, $80$) & 20.05 (24s, $70$) & \textbf{23.82} (10s, $67$)\\
  25	   & 17.01 (28s, $75$)  & 16.88 (30s, $80$)  & 16.26 (20s, $75$) & 17.73 (24s, $70$) &\textbf{17.92} (10s, $65$)\\
  30	  &  10.77 (28s, $90$)  & 11.71 (30s, $90$)  & 11.93 (20s, $73$) & 11.05 (24s, $70$) &\textbf{12.67} (10s, $62$)\\
  	\hline
    \end{tabular}
    \caption{Comparison of different algorithms in terms of noise level ($\sigma_n$) for the $Cameraman$ gray scale image. The results are given in terms of best possible PSNR (computational time in seconds, maximum iterations). Each scheme is terminated if the maximum number of iterations exceeded $500$ or when the duality gap is less than $\mathcal{R}(u,b) \leq 10^{-6}$.}\label{t:iters}
\end{table}

First comparative example in Fig.~\ref{fig:cameraman} compares the restoration results for the noisy $Cameraman$ gray scale image from Fig.~\ref{fig:camoriginal} (b). As can be seen, adaptive Huber function performs better than the classical TV and Tikhonov schemes. Moreover, improvement in PSNR is $>5dB$ (see Fig.~\ref{fig:cameraman} (d)) in different noise levels which indicates the success of our scheme in terms of noise removal. Table~\ref{t:gap} shows the number of iterations taken by different optimization schemes for solving the discrete regularization scheme~\eqref{E:discf} with respect to the relative duality gap error~\eqref{E:gap} as a stopping criteria. The split Bregman based implementation outperforms all the other schemes by reducing the relative duality gap within very few iterations. Next, Table~\ref{t:iters} provides a comparison of PSNR (time in seconds, maximum iterations) for different noise levels and for different optimization schemes for the noisy $Cameraman$ image. The experiments were performed on a Mac Pro Laptop with 2.3GHz Intel Core i7 processor, 8Gb memory and MATLAB R2012a was used for visualizations. The split Bregman minimization outperforms all the related schemes in terms of PSNR (dB) as well as in timing as can be seen from the table. Similar analysis for the image deblurring and deconvolution requires a delicate analysis of the boundary conditions~\cite{ShiChangTVBCANM08} and is treated elsewhere. Other avenues of exploration are treating higher order models~\cite{WuYangPangfourthfixedpointANM12,JianghighshrinkageANM12}, multi grid~\cite{SpitaleriMarchmultigridANM01} and FEM~\cite{KacurMikulaFEManiso95} based schemes and their convergence analysis.
\section{Conclusion}\label{sec:conc}

In this paper we considered adaptive Huber type regularization function based image restoration scheme. By using discrete split Bregman scheme we proved the convergence to continuous formulation. Experimental results on real images are given to illustrate the results presented. Compared with other schemes the splitting based scheme provides faster convergence as well as good restoration results. The scheme can be extended to handle multispectral images by using inter-channel correlations~\cite{PSb10,Prasathi11,PrasathJCMPalSPL13} and this defines our future work.

\bibliographystyle{plain}
\bibliography{myrefs}
\end{document}